\documentclass[sn-mathphys-num]{sn-jnl}

\usepackage{amsmath,amssymb,amsfonts}%
\usepackage{amsthm}%
\usepackage{mathrsfs}%
\usepackage[title]{appendix}%
\usepackage{textcomp}%
\usepackage{manyfoot}%
\usepackage{booktabs}%
\usepackage{algorithm}%
\usepackage{algorithmicx}%
\usepackage{algpseudocode}%
\usepackage{listings}%

\usepackage{amsmath}
\usepackage{amsfonts}
\usepackage{fontawesome5}

\usepackage{multirow}
\usepackage{graphicx}
\usepackage{makecell}
\usepackage{xcolor}
\usepackage{colortbl}
\usepackage{hhline}

\definecolor{olive}{rgb}{0.6, 0.6, 0.2}
\definecolor{sand}{rgb}{0.8666666666666667, 0.8, 0.4666666666666667}
\definecolor{wine}{rgb}{0.5333333333333333, 0.13333333333333333, 0.3333333333333333}

\definecolor{deblue}{RGB}{11,132,147}
\definecolor{ocra}{RGB}{204, 119, 34}
\usepackage{enumitem}
\usepackage{tikz}
\newcommand{\fcircle}[2][red,fill=red]{\tikz[baseline=-0.5ex]\draw[#1,radius=#2] (0,0.03) circle ;}

\usepackage{mdframed} 

\newmdtheoremenv[
    backgroundcolor=lightgray!10, 
    linecolor=black, 
    linewidth=0.5pt, 
    roundcorner=5pt, 
    innertopmargin=8pt, 
    innerbottommargin=8pt, 
    skipabove=10pt, 
    skipbelow=10pt 
]{theorem}{Theorem}

\DeclareUnicodeCharacter{2207}{$\delta$}

\begin{document}

\title[Article Title]{You KAN Do It in a Single Shot: Plug-and-Play Methods with Single-Instance Priors}


\author[1]{\fnm{Yanqi} \sur{Cheng}}\email{yc443@cam.ac.uk}

\author[1]{\fnm{Carola-Bibiane} \sur{Schönlieb}}\email{cbs31@cam.ac.uk}

\author[2]{\fnm{Angelica I} \sur{Aviles-Rivero}}\email{aviles-rivero@tsinghua.edu.cn}

\affil[1]{\orgdiv{Department of Applied Mathematics and Theoretical Physics}, \orgname{University of Cambridge}}

\affil[2]{\orgdiv{Yau Mathematical Sciences Center}, \orgname{Tsinghua University}}

\abstract{
The use of Plug-and-Play (PnP) methods has become a central approach for solving inverse problems, with denoisers serving as regularising priors that guide optimisation towards a clean solution. In this work, we introduce KAN-PnP, an optimisation framework that incorporates Kolmogorov-Arnold Networks (KANs) as denoisers within the Plug-and-Play (PnP) paradigm. KAN-PnP is specifically designed to solve inverse problems with single-instance priors, where only a single noisy observation is available, eliminating the need for large datasets typically required by traditional denoising methods. We show that KANs, based on the Kolmogorov-Arnold representation theorem, serve effectively as priors in such settings, providing a robust approach to denoising.
We prove that the KAN denoiser is Lipschitz continuous, ensuring stability and convergence in optimisation algorithms like PnP-ADMM, even in the context of single-shot learning. Additionally, we provide theoretical guarantees for KAN-PnP, demonstrating its convergence under key conditions: the convexity of the data fidelity term, Lipschitz continuity of the denoiser, and boundedness of the regularisation functional. These conditions are crucial for stable and reliable optimisation.
Our experimental results show, on super-resolution and joint optimisation, that KAN-PnP outperforms exiting  methods, delivering superior performance in single-shot learning with minimal data. The method exhibits strong convergence properties, achieving high accuracy with fewer iterations. 
}


\maketitle

\section{Introduction}\label{sec1}
Inverse problems are widely studied in scientific computing and imaging, where the goal is to recover an unknown signal $x \in \mathbb{R}^N$ from indirect and often noisy measurements $y \in \mathbb{R}^M$. These measurements are typically modelled as: 
\begin{equation}
    y = A(x) + \epsilon,
\end{equation}
where $A : \mathbb{R}^N \to \mathbb{R}^M$ is a forward operator representing the measurement process, and $\epsilon$ accounts for noise or perturbations in the system. Depending on the application, $A$ may be linear, such as in magnetic resonance imaging (MRI)~\cite{fessler2010model} or computed tomography (CT)~\cite{elbakri2002statistical}, or nonlinear, such as in phase retrieval or optical imaging.

Due to challenges like noise contamination, incomplete data (e.g., undersampling), or the ill-posed nature~\cite{tikhonov1987ill, engl1996regularization, chambolle2016introduction} of the inverse operator $A^{-1}$, direct inversion is generally unstable or infeasible. To address this, regularisation methods are employed to incorporate prior knowledge about the desired solution $x$. Mathematically, this leads to solving a variational problem:
\begin{equation}\label{reg_Framework}
\min_{x \in \mathbb{R}^N} D(x) + \lambda R(x),
\end{equation}
where $D(x)$ is the data-fidelity term that ensures consistency with the measurements $y$, $R(x)$ is the regularisation term encoding prior knowledge about $x$, and $\lambda$ balances the two terms.

The proximal operator is a key mathematical tool~\cite{gribonval2020characterization} in solving regularised problems of the form given in~\eqref{reg_Framework}. For a function $R(x)$, its proximal operator is defined as:
\begin{equation}
\text{Prox}_{\gamma R}(v) = \arg\min_{x \in \mathbb{R}^N} \left\{ R(x) + \frac{1}{2\gamma} \|x - v\|^2 \right\},
\end{equation}
where $\gamma > 0$ is a step size parameter. Intuitively, the proximal operator can be interpreted as a denoising step that produces a regularised estimate of $v$ by minimising the sum of $R(x)$ and a quadratic penalisation term.

Many iterative optimisation algorithms, such as proximal gradient methods and the alternating direction method of multipliers (ADMM)~\cite{gabay1976dual}, leverage the proximal operator to solve problems with composite objectives. For instance, in ADMM, the proximal operator of the regularisation term $R(x)$ is applied iteratively to refine the solution.

Traditionally, $R(x)$ is chosen to be a mathematically explicit regularisation function, such as total variation, wavelet sparsity, or nuclear norm. However, modern approaches have shifted toward \textit{implicit regularisation}, where the proximal operator is replaced with a learned or pre-defined denoiser $H_\sigma(x)$. This substitution leads to the Plug-and-Play (PnP) framework, introduced by Venkatakrishnan et al.~\cite{venkatakrishnan2013plug}, which exploits the equivalence between regularised denoising and the proximal operator.
In the PnP framework, the optimisation problem is reformulated using a denoiser $H_\sigma(x)$ as a surrogate for $\text{Prox}_{\gamma R}$. For example, the ADMM iteration in PnP is given by:
\begin{equation}
\begin{aligned}
    x^{k+1} &= H_\sigma(z^k - u^k), \\
    z^{k+1} &= \text{Prox}_{\mu^{-1} D}(x^{k+1} + u^k), \\
    u^{k+1} &= u^k + x^{k+1} - z^{k+1}.
\end{aligned}
\end{equation}
where $\sigma$ controls the denoising strength of $H_\sigma$, and $\mu$ is a parameter related to the data-fidelity term. 

Building on this foundation, PnP has evolved significantly since its introduction by~\cite{venkatakrishnan2013plug}. Early approaches demonstrated its effectiveness across diverse imaging tasks~\cite{meinhardt2017learning}, while subsequent studies~\cite{ryu2019plug,sun2019online,hurault2022proximal,hurault2021gradient} refined the framework by enhancing stability, convergence, and applicability. Additionally, works such as~\cite{teodoro2018convergent,yuan2020plug,zhang2017learning,ono2017primal,sun2019online} extended its reach to more complex imaging scenarios. A notable milestone was the development of TFPnP (Tuning-Free Plug-and-Play)~\cite{wei2020tuning}, which eliminated the need for parameter tuning, simplifying practical deployment. 

The core of the Plug-and-Play (PnP) framework lies in the denoiser, $H_\sigma$, which serves as an implicit regulariser for the reconstruction process. This denoiser is typically trained using a large amount of data, allowing it to learn rich priors that generalise well to diverse instances of the underlying task. 
PnP frameworks have been developed utilising a variety of denoisers. Classical methods like BM3D~\cite{dabov2007image} and TV~\cite{jordan1881series} remain popular due to their robustness and minimal reliance on training data. In the meantime, the rise of deep learning-based denoisers~\cite{meinhardt2017learning,zhang2017learning,laumont2022bayesian} has significantly advanced the field by capturing complex data patterns, such as DnCNN~\cite{zhang2017beyond} or U-Net~\cite{ronneberger2015u} achieve their performance by leveraging extensive datasets to capture the statistical properties of natural images or specific domains. Recent works on Weakly Convex Ridge Regularisers (WCRR-NNs)~\cite{goujon2024learning} proposes learning explicit regularisation functionals \( R(x) \) such that the inverse problem remains (weakly) convex. These methods operate in a variational setting and typically require access to a dataset for training. However, this reliance on large-scale training datasets poses a significant challenge in domains where data availability is limited.

A key question thus arises: \textit{how can PnP frameworks be adapted to solve inverse problems effectively when only a single instance or minimal data is available?} Addressing this question requires rethinking the fundamental reliance of PnP methods on large-scale training datasets and instead developing approaches that can utilise the limited data at hand efficiently. The recent application of Deep Image Prior (DIP)~\cite{ulyanov2018deep} within the Plug-and-Play (PnP) framework~\cite{sun2021plug} represents a noteworthy development. However, in practice, DIP-PnP still depends on training the model with multiple instances. Our focus, in contrast, is on minimising the data required for prior training in the PnP setting—specifically, training on a single instance. To the best of our knowledge, the only existing work that aligns with ours is that of Cheng et al.~\cite{cheng2024singleshot}, which introduces the single-shot paradigm within the Plug-and-Play context. We highlight key differences between their approach and ours:
Firstly, their prior is based on implicit neural representations, whereas ours leverages the Kolmogorov-Arnold representation theorem. Moreover, unlike their work, we provide theoretical results on both the denoiser's properties and the convergence of the entire scheme.

\medskip
\textbf{Constributions.} We introduce a novel optimisation framework called KAN-PnP, which leverages Kolmogorov-Arnold Networks (KANs) as denoisers within the Plug-and-Play (PnP) optimisation paradigm. This method addresses the challenge of solving inverse problems in the context of single-instance priors, where only a single noisy observation is available. Our approach offers a significant departure from traditional denoising methods that rely on large datasets, providing a robust solution for cases with minimal data. The key contributions of this work are as follows:

\smallskip
\noindent
\fcircle[fill=deblue]{2pt} We present the first framework to incorporate Kolmogorov-Arnold Networks (KANs)~\cite{liu2024kan} as denoisers in the PnP paradigm, focusing on solving inverse problems with single-instance priors: \begin{itemize}[noitemsep, nolistsep] 
\item[--]  We demonstrate how KANs, based on the Kolmogorov-Arnold representation theorem, can effectively serve as priors in settings where only one noisy observation is available, bypassing the need for large datasets typically required by traditional denoising techniques. 
\item[--]  We prove that the KAN denoiser is Lipschitz continuous, ensuring its stability and guaranteeing convergence in optimisation algorithms like PnP-ADMM (Theorem~\ref{thm1}). 
\item[--] The framework is designed to handle single-shot learning, where the denoiser is adapted to a specific image instance, allowing it to perform well in real-world scenarios with limited data. 
\end{itemize}
\fcircle[fill=deblue]{2pt}  We establish the convergence of the KAN-PnP algorithm, proving that it converges to a stationary point under certain conditions (Theorem~\ref{thm2}):  I) Convexity of the Data Fidelity Term: Ensures the well-posedness of the optimisation problem and the solvability of the data fidelity subproblem. II) Lipschitz Continuity of the Denoiser: Guarantees the stability of the denoising step, ensuring that the method behaves well during optimisation. III) Boundedness of the Regularisation Functional: Prevents unbounded behaviour in the regularisation term, ensuring a stable optimisation process. 

\smallskip
\noindent
\fcircle[fill=deblue]{2pt}  We validate the effectiveness of KAN-PnP through extensive experiments on various inverse problems. We demonstrate that  KAN-PnP significantly outperforms existing methods, when trained with a single instance observation. Moreover, we show that our method shows excellent convergence properties, reliably approaching a stationary point with minimal iterations. Finally, We compare KAN-PnP with state-of-the-art denoising methods and show that it provides superior performance for single and multi-operator inverse problems including super-resolution and joint optimisation.

\section{Proposed Method}

We remain to the reader that we consider the following class of optimisation problems from the Plug-and-Play perspective:
\begin{equation}
    \min_{x, z} f(x) + \lambda R(H_\sigma(z)), \quad \text{subject to} \quad x = z
\end{equation}
where \( f(x) \) is the data fidelity term, \( \lambda \) is a regularisation parameter, and \( R(x) \) is the regularisation functional implicitly defined through a network  denoiser \( H_\sigma(x) \). This framework generalises a wide range of inverse problems. We note that (5) should be understood as a conceptual variational formulation underpinning the Plug-and-Play scheme. While $R(H_\sigma(z))$ is not used explicitly in the algorithm, it allows us to interpret the denoising step as a surrogate proximal operator of an implicit regulariser $R$. 

A critical component of the Plug-and-Play methods  is the choice of the denoiser. In traditional denoising techniques, methods such as non-local means (NLM)~\cite{buades2005non}, BM3D~\cite{dabov2007image}, and deep learning-based approaches like DnCNN~\cite{zhang2017beyond} have been widely used. These methods, although effective, \textit{often rely on large datasets for training, which limits their application in scenarios where only a single noisy observation is available}. In this section, we detail a novel framework called KAN-PnP, which uses a single-instance prior.

\subsection{KAN as the Ideal Single-Instance Prior}
In our framework, \textit{we focus on a single-shot learning setting}, where we only have a single noisy observation of the underlying clean signal. This presents a challenge, as traditional denoisers may struggle in such settings, where the noise characteristics are not fixed across samples.
To address this, we propose a novel prior drawing from Kolmogorov-Arnold Networks (KANs) as the denoiser. KANs are a powerful class of neural networks based on the Kolmogorov-Arnold representation theorem, which ensures that any continuous function can be approximated by a finite composition of smooth, univariate functions. This ability to smoothly approximate complex mappings makes KANs particularly well-suited for denoising tasks, as they naturally preserve the underlying structure of the signal while removing noise. KANs introduce a form of structured inductive bias through kernel-based interpolation over learnable anchors, which induces local smoothness and adaptivity — properties that are well-aligned with the statistical structure of images. Unlike MLP-based INRs, which are globally connected and prone to spectral bias (favoring low-frequency content), KAN’s piece- wise smooth interpolation better captures high-frequency details and edges, which are essential for image restoration tasks (imaging inverse problems). In the single- shot setting, where no external supervision is available, this inductive bias acts as a strong prior that regularises the learning dynamics and helps avoid overfitting to noise.
Kolmogorov-Arnold Networks (KANs) are a class of neural networks based on the Kolmogorov-Arnold representation theorem~\cite{Kolmogorov1961,kolmogorov1957representations,braun2009constructive}, which guarantees that any continuous function on a multidimensional domain can be approximated by a finite composition of continuous functions.  KAN is defined as follows.
Let \( x \in \mathbb{R}^d \) be the input vector, and let \( y \in \mathbb{R} \) be the output. A KAN approximates a continuous function \( f(x) \) by composing several layers of univariate basis functions. Specifically, a KAN can be expressed as:
\[
f(x) = \sum_{i=1}^{m} \sum_{j=1}^{n} \alpha_{ij} \phi_{ij} (x_i),
\]
where \( \phi_{ij}(x_i) \) are univariate functions (typically chosen to be continuous and smooth, such as B-splines or polynomial functions). Moreover,  \( \alpha_{ij} \) are the weights (parameters) of the network, and \( x_i \) represents the components of the input vector \( x \). Also, \( m \) and \( n \) are the number of hidden layers and basis functions, respectively.
Each hidden layer \( \phi_{ij}(x_i) \) is constructed from a basis function that is chosen to provide smooth approximations, and the overall network combines these functions to model the desired output.

\smallskip
\textbf{KAN-Plug-and-Play Method.} In the context of Plug-and-Play Methods, and in particular in this work PnP-ADMM, the denoiser serves as the regularising prior, which guides the optimisation process toward a clean solution. One critical property of the denoiser is Lipschitz continuity~\cite{hurault2022proximal}, which ensures stability of the denoiser. This property is essential for the convergence of iterative algorithms like PnP-ADMM, where stability and well-behaved priors are necessary for guaranteeing the algorithm’s success.

\begin{theorem}[Lipschitz Continuity of the KAN Denoiser]\label{thm1}
Let $H(x)$ represent the KAN denoiser defined as a composition of layers:
\begin{equation}
H(x) = \Phi_L \circ \Phi_{L-1} \circ \dots \circ \Phi_1(x),
\end{equation}
where each layer $\Phi_l$ is a matrix of univariate B-spline functions $\phi_{l,j,i}(x)$ with bounded derivatives. Then $H(x)$ is Lipschitz continuous, i.e., there exists a constant $L > 0$ such that:
\begin{equation}
\|H(x) - H(y)\| \leq L \|x - y\|, \quad \forall x, y \in \mathbb{R}^n.
\end{equation}
\end{theorem}

\begin{proof}
We begin by analysing the Lipschitz continuity of each individual layer in the KAN architecture. Recall that the activation value at layer $l+1$ is computed as:
\begin{equation}
x_{l+1, j} = \sum_{i=1}^{n_l} \phi_{l,j,i}(x_{l,i}),
\end{equation}
where $x_{l,i}$ represents the output of neuron $i$ in layer $l$, and $\phi_{l,j,i}$ denotes the univariate B-spline function that connects neuron $i$ in layer $l$ to neuron $j$ in layer $l+1$. The pre-activation values are aggregated using a summation operation, and each $\phi_{l,j,i}$ is differentiable and piecewise smooth.

Let $\Phi_l$ denote the mapping from $x_l$ to $x_{l+1}$. Each univariate spline $\phi_{l,j,i}(x)$ is parametrised using B-splines, whose maximum derivative depends on the number of control points and the smoothness order $k$. Assuming the splines are constructed with $G$ grid points and bounded coefficients, the maximum derivative $L_{\phi_{l,j,i}}$ can be explicitly bounded as:
\begin{equation}
L_{\phi_{l,j,i}} \leq C G^{k-1},
\end{equation}
where $C$ is a constant dependent on the initialisation scale and spline parameters.

Since the post-activation value at neuron $j$ in layer $l+1$ is a summation of the outputs of splines, we can bound the Lipschitz constant of each neuron as:
\begin{equation}
L_{\Phi_l, j} = \sum_{i=1}^{n_l} L_{\phi_{l,j,i}}.
\end{equation}
The Lipschitz constant for the entire layer $\Phi_l$ is therefore:
\begin{equation}
L_{\Phi_l} \leq \max_{j} \sum_{i=1}^{n_l} C G^{k-1}.
\end{equation}
The overall denoiser $H(x)$ is a composition of $L$ such layers:
\begin{equation}
H(x) = \Phi_L \circ \Phi_{L-1} \circ \dots \circ \Phi_1(x).
\end{equation}
By the chain rule for Lipschitz continuity, the Lipschitz constant of the composite function is given by the product of the Lipschitz constants of the individual layers:
\begin{equation}
L_H = \prod_{l=1}^L L_{\Phi_l}.
\end{equation}
Substituting the bound for $L_{\Phi_l}$ into this expression, we have:
\begin{equation}
L_H \leq \prod_{l=1}^L \left( \max_{j} \sum_{i=1}^{n_l} C G^{k-1} \right).
\end{equation}
Since each layer $\Phi_l$ has a finite Lipschitz constant $L_{\Phi_l}$, their composition $H(x)$ is also Lipschitz continuous with a finite constant $L_H$. This completes the proof.
\end{proof}

\textbf{On the Theorem 1 and its Implications.}
 While Lipschitz continuity of neural networks is well-known, we emphasise that our result is tailored to the Kolmogorov-Arnold Network (KAN) architecture, which differs significantly from standard ReLU-based networks. The use of smooth univariate B-spline basis functions requires a more explicit analysis to ensure global Lipschitz continuity, especially in the context of single-instance priors, where no data-driven regularisation is available.
Although the derived Lipschitz bound in Theorem 1 may be conservative, it offers a principled way to control the smoothness and stability of the KAN denoiser via architectural parameters (e.g., number of grid points $G$, spline smoothness $k$, and layer width). This enables practitioners to design networks with desirable properties, including potential non-expansiveness ($L_H \leq 1$), which is beneficial for convergence in Plug-and-Play schemes.

\subsection{Key Assumptions for KAN-PnP}
Having established that the KAN denoiser $H_\sigma$ is Lipschitz continuous, we now outline the key assumptions required to ensure the convergence of the Plug-and-Play framework. In particular, we focus on  PnP-ADMM framework in the context of single-shot learning setting. Unlike traditional settings where the denoiser is trained on a large dataset, single-shot learning involves tailoring $H_\sigma$ to a single image instance. This adaptation introduces unique challenges in satisfying the assumptions, which we address below:

\medskip
\fcircle[fill=deblue]{2pt}\textbf{ Convexity of the Data Fidelity Term.} 
The data fidelity term $f(x)$ must be convex, proper, and lower semi-continuous. In this work, we use the squared $\ell_2$-norm, defined as:
    \begin{equation}
        f(x) = \|Ax - y\|_2^2,
    \end{equation}
which is a standard choice in inverse problems due to its simplicity and convexity. This ensures that the subproblem for $z^{k+1}$ in the ADMM iterations is well-posed and solvable. The use of $f(x)$ remains valid regardless of whether $H_\sigma$ is trained on a large dataset or a single image.

\fcircle[fill=deblue]{2pt} \textbf{Lipschitz Continuity of the Denoiser.} 
As demonstrated in the previous section, the KAN denoiser $H_\sigma$ is Lipschitz continuous with a constant $L_H$. This guarantees the stability of the denoising step, making it suitable as a surrogate for the proximal operator in the PnP framework. 
In single-shot learning, ensuring Lipschitz continuity requires architectural constraints within the KAN framework. The use of bounded spline parameters and regularisation within the network ensures controlled outputs, even when the denoiser is adapted to a single image.

\fcircle[fill=deblue]{2pt}\textbf{ Boundedness of the Regularisation Functional.} 
The regulariser $R(x)$ associated with the KAN denoiser is defined implicitly via the proximal operator formulation:
    \begin{equation}
        H_\sigma(x) = \arg\min_{z \in \mathbb{R}^n} \left\{ R(z) + \frac{1}{2\sigma^2} \|z - x\|^2 \right\}.
    \end{equation}
This implies that $R(x)$ is bounded below due to the quadratic penalty term, which prevents $R(x)$ from decreasing. However, in the single-shot setting, the behavior of $R(x)$ depends on how $H_\sigma$ is constructed and adapted. To ensure boundedness:

\noindent
\textcolor{deblue}{\faAngleDoubleRight}  Architectural smoothness (e.g., using B-splines with bounded derivatives) in the KAN denoiser suppresses non-coercive behaviour. \\
\textcolor{deblue}{\faAngleDoubleRight}  The quadratic term in the proximal operator formulation helps dominate \(R(z)\) for large $z$, further ensuring boundedness.

\noindent
We remark that (16) defines \( H_\sigma \) as the proximal operator of a regularisation functional \( R \). However, such an \( R \) exists only if \( H_\sigma \) is a firmly nonexpansive operator — a condition not guaranteed in the learned denoising setting. In our framework, we treat this formulation as a conceptual model aligned with recent works (e.g., \cite{ryu2019plug}), and do not assume the explicit existence or convexity of \( R \). Our convergence results rely on the Lipschitz continuity of \( H_\sigma \), convexity of the data fidelity term, and boundedness of the surrogate regularisation.

\smallskip
\fcircle[fill=deblue]{2pt}\textbf{ Differentiability of the Implicit Regulariser.} 
The implicit regulariser $R(x)$ derived from $H_\sigma$ must have a Lipschitz continuous gradient. This condition can be satisfied if $H_\sigma(x)$ is differentiable and smooth, which follows from the construction of the KAN denoiser using B-splines. In single-shot learning, we ensure the smoothness of $H_\sigma$, avoiding overfitting artifacts or discontinuities introduced during adaptation to the single image.
We emphasise that while (16) conceptually associates the denoiser \( H_\sigma \) with a regularisation functional \( R \), we do not assume that \( R \) is differentiable. However, since the KAN denoiser \( H_\sigma \) is constructed from smooth B-spline basis functions, it is differentiable. This differentiability enables practical use of updates in PnP-ADMM, and contributes to stable convergence behavior. Nonetheless, the differentiability of \( H_\sigma \) does not imply the existence or differentiability of an underlying functional \( R \).

\begin{figure}[t!]
\centering
\includegraphics[width=\textwidth]{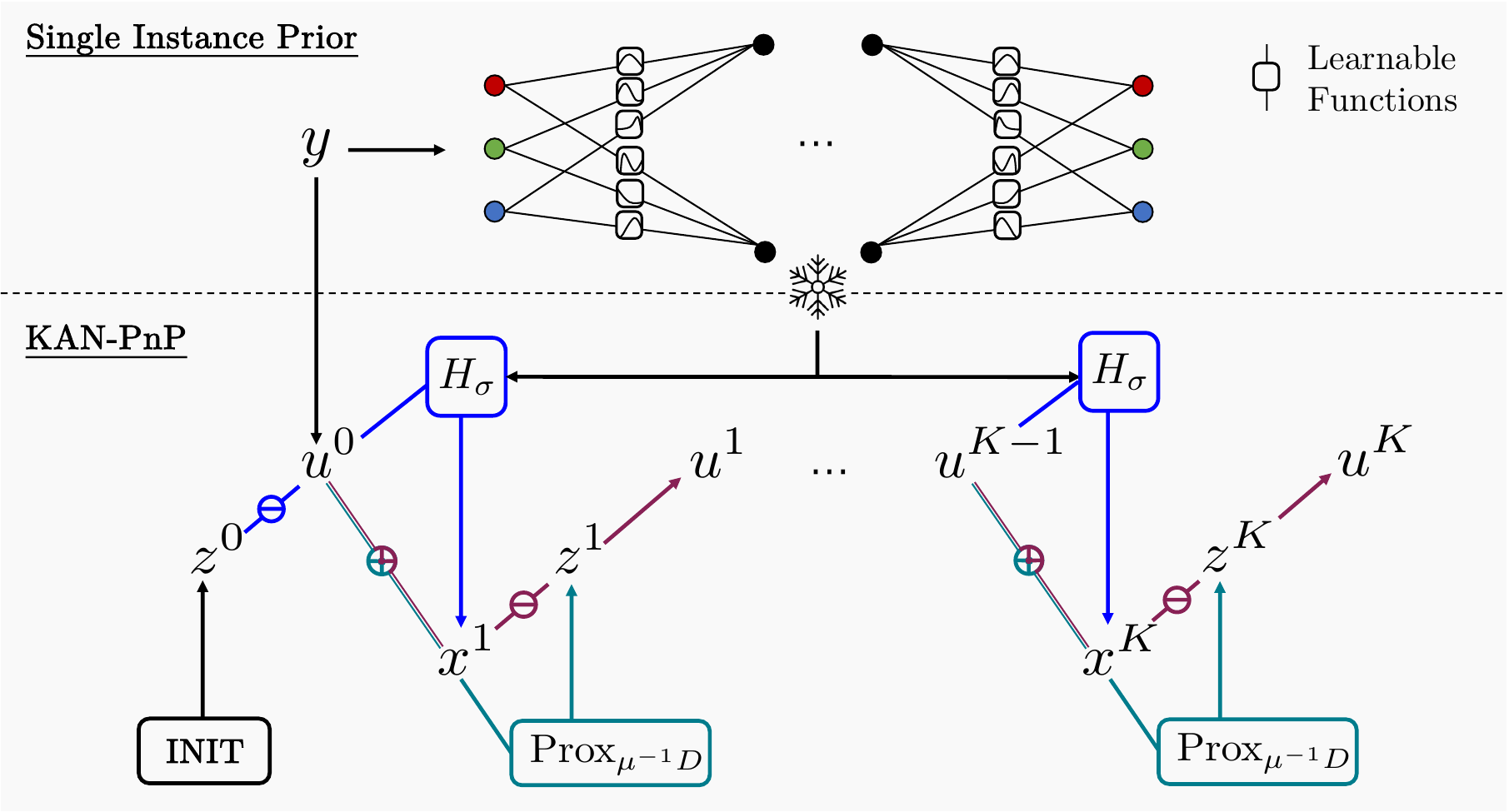}
\caption{The proposed KAN-PnP framework, comprising the training of a single-instance prior and its integration into an iterative optimisation algorithm.}\label{teaser}
\end{figure}

\subsection{Convergence of KAN-PnP with Single-Instance Prior}
In this section, we analyse the convergence of our KAN-PnP framework for solving the inverse problem, where the regulariser is replaced by the KAN denoiser \(H_\sigma\). We recall that our problem is:
\[
\min_{x, z} f(x) + \lambda R(H_\sigma(z)) \quad \text{subject to} \quad x = z,
\]
where \( f(x) \) is the data fidelity term and \( R(H_\sigma(z)) \) represents the implicit regularisation term provided by the KAN denoiser.

The KAN-PnP (see Figure~\ref{teaser}) with ADMM  proceeds by iteratively solving the following three subproblems:
\begin{align}\label{Kan-pnp-admm}
\textbf{I. Denoising Step}: & \ \textcolor{blue}{x^{k+1} = H_\sigma(z^k - u^k)},  \\
\textbf{II. Data Fidelity Step}: & \ \textcolor{deblue}{z^{k+1} = \arg\min_z \left( f(z) + \frac{\mu}{2} \|z - (x^{k+1} + u^k)\|^2 \right)},\nonumber \\
\textbf{III. Dual Update}: & \ \textcolor{wine}{u^{k+1} = u^k + x^{k+1} - z^{k+1}}. \nonumber
\end{align}

In this section, we show that under the assumptions of Lipschitz continuity, convexity of \( f(x) \), and boundedness of the regulariser \( R(x) \),  KAN-PnP algorithm converges to a stationary point.

\smallskip
\faMedapps \textbf{ Key Assumptions Recap.} To ensure convergence of the KAN-PnP framework with the KAN denoiser, we rely on the following assumptions: 
\begin{enumerate} 
\item Convexity of \( f(x) \), which ensures well-posedness of the data fidelity subproblem. 
\item Lipschitz continuity of the KAN denoiser $H$ ensures that the denoising step is stable and does not grow arbitrarily with input perturbations. While this does not imply non-expansiveness (i.e., 1-Lipschitz continuity), convergence can still be guaranteed under this weaker assumption, as discussed in~\cite{ryu2019plug}.
\item Differentiability of the regulariser \( R(x) \), if required for  updates, which is satisfied under the smoothness of \( H_\sigma \). 
\end{enumerate}

\medskip
\begin{theorem}[Fixed-Point Convergence of KAN-PnP] \label{thm2}
Under the assumptions:
\begin{enumerate}[label=(\roman*)]
    \item $f: \mathbb{R}^n \to \mathbb{R}$ is proper, closed, and convex,
    \item $H_\sigma: \mathbb{R}^n \to \mathbb{R}^n$ is Lipschitz continuous with constant $L_H$ (as established in Theorem~1),
    \item The penalty parameter $\mu > L_H$ is constant across iterations,
\end{enumerate}
the PnP-ADMM iterates $\{(x^k, z^k, u^k)\}$ defined in (17) converge to a fixed point $(x^*, z^*, u^*)$ satisfying:
\[
x^* = H_\sigma(z^* - u^*), \quad z^* = \arg\min_z \left( f(z) + \frac{\mu}{2} \|z - (x^* + u^*)\|^2 \right), \quad x^* = z^*.
\]\vspace{-0.5cm}
\end{theorem}

\vspace{0.2cm}

\paragraph{Proof Sketch.}
We analyse the convergence of the KAN-PnP ADMM iterations using the fixed-point framework introduced in recent Plug-and-Play literature~\cite{ryu2019plug, hurault2022proximal}. This approach does not assume the existence of an explicit energy functional involving the denoiser, but instead analyses the behavior of the iterative updates as an operator composition acting on the optimisation variables.
The denoising step $x^{k+1} = H_\sigma(z^k - u^k)$ applies a Lipschitz continuous operator $H_\sigma$ with constant $L_H$ (as guaranteed by Theorem~1). The $z$-update step:
\[
z^{k+1} = \arg\min_z \left( f(z) + \frac{\mu}{2} \|z - (x^{k+1} + u^k)\|^2 \right),
\]
is equivalent to the proximal operator of the convex function $f$, evaluated at $x^{k+1} + u^k$. Since $f$ is convex and proper, the proximal operator $\text{Prox}_{f/\mu}$ is firmly nonexpansive.

The dual update $u^{k+1} = u^k + x^{k+1} - z^{k+1}$ is a standard ADMM step enforcing consensus between $x$ and $z$. When combining these three steps, the overall ADMM iteration can be interpreted as a fixed-point iteration of a composed operator acting on $(x, z, u)$.
Under the assumption that $\mu > L_H$, it follows from standard results in fixed-point theory that the composition of the Lipschitz operator $H_\sigma$ and the firmly nonexpansive proximal operator defines an averaged operator with a contraction factor strictly less than one~\cite{ryu2019plug}. By the Banach fixed-point theorem, this guarantees that the sequence $\{(x^k, z^k, u^k)\}$ converges to a unique fixed point $(x^*, z^*, u^*)$ satisfying:
\[
x^* = H_\sigma(z^* - u^*), \quad
z^* = \text{Prox}_{f/\mu}(x^* + u^*), \quad
x^* = z^*.
\]

These conditions characterise a consistent fixed point of the algorithm, even though the objective function is implicit. This result justifies the convergence of the KAN-PnP method under mild and verifiable assumptions, without requiring knowledge of the underlying regulariser associated with $H_\sigma$.
\qed

\textbf{Remark.}  
The convergence result in Theorem~2 adapts the fixed-point framework established in recent Plug-and-Play literature~\cite{ryu2019plug, hurault2022proximal}. While the overall structure follows these works, our analysis is specialised to the single-instance setting, where the denoiser \( H_\sigma \) is trained directly from a noisy observation. This introduces unique architectural properties (e.g., differentiability from B-splines) and excludes any access to external datasets or explicit regularisers. Our proof reflects this structure and confirms convergence under mild assumptions on the denoiser and fidelity term.

\section{Experimental Results \& Discussion} \label{experiment_setting}
\label{sec:experiment}
In this section, we evaluate the performance of our proposed KAN-PnP framework across challenging image reconstruction tasks, including super-resolution and joint optimisation for demosaicing and deconvolution. Our approach leverages a \textit{single-instance prior}, wherein the denoiser is trained directly on the observed image itself, rather than on a dataset of similar images.  The goal of these experiments is to assess the effectiveness of our approach compared to other single-instance denoisers and the conventional methods, including pretrained denoisers and total variation techniques within the Plug-and-Play framework. This training scheme is designed to exploit the unique characteristics of the target image, enhancing reconstruction quality without requiring extensive external datasets. This rationale motivates our comparison with traditional methods trained in a conventional manner.

For the experimental comparison, we followed the dual resizing strategies to preprocess the image. The images are sourced with Creative Commons Licenses which are resized to $512 \times 384$, and the selected data in~\cite{bevilacqua2012low, zeyde2012single} were tested. 
The experiment were conducted on NVIDIA A100 GPU with 80GB RAM. The Plug-and-Play phase ensued with the $\bigtriangledown$-Prox toolbox~\cite{lai2023prox}, adopting the default settings for all tasks.
\textbf{Why Natural Images?}  
While our method is applicable to diverse data types, we focus on natural images in this work for several reasons. First, natural image benchmarks (e.g., Set5, Set14) are widely used in foundational works on single-instance priors (e.g., DIP~\cite{ulyanov2018deep} and  Cheng et al.~\cite{cheng2024singleshot}), enabling direct comparison with prior art. Second, these datasets offer well-understood structure and standardized metrics, which help isolate the methodological contribution of the learned prior. Third, since our aim is to investigate core behavior of single-shot Plug-and-Play priors — a relatively unexplored domain — natural images provide a valuable, interpretable setting for proof-of-concept evaluation.

\fcircle[fill=deblue]{2pt} \textbf{Training Scheme.} The initial pre-training phase for the implicit neural coordinate network including the implicit neural representation (INR) and the implicit Kolmogorov-Arnold Network (KAN) denoising prior incorporated Gaussian noise with a standard deviation of $0.1$.
The single-shot denoising prior training spanned $100$ iterations using a INR network composed of $2$ hidden layers, each with $64$ features, while the KAN network is composed of $3$ hidden layers with $\{128, 32, 16\}$ features with grid $5$.
The learning rate was maintained at $0.001$. Subsequently, image reconstruction was performed over $5$ ADMM iterations, employing dynamically adjusted noise levels and a penalty parameter that followed a logarithmic descent across the steps.

\fcircle[fill=deblue]{2pt} \textbf{Evaluation Protocol.} To ensure a robust comparison, we evaluated the Noise2Self pre-training approach~\cite{krull2019noise2void} on three networks, each undergoing $100$ iterations with a learning rate of $0.01$. The FFDNet~\cite{zhang2018ffdnet} model is designed with $8$ hidden layers and $64$ feature maps per layer, while the UNet~\cite{ronneberger2015u} architecture featured a $4$-level downsampling scheme and began with $32$ feature maps in its initial convolutional layer. For consistency, the optimisation setup mirrored the configuration used in our proposed approach. We compare these results against both the Noise2Self approach and implicit neural MLP, which we denoted as INR, as well as our proposed implicit KAN method. The evaluation of the methods were measured by Peak Signal-to-Noise Ratio (PSNR) and Structural Similarity Index (SSIM), where PSNR and SSIM scores signal are higher when reconstruction quality improved. \textbf{On Parameter Counts.}
While the KAN and INR architectures differ in parameter count, we do not enforce equality since they encode fundamentally different inductive biases. KANs exploit local interpolation priors, while INRs rely on global MLP representations. For this reason, we evaluate models based on reconstruction quality rather than parameter parity.

\subsection{ KAN-PnP in High-Stakes Image Reconstruction}
Our proposed KAN-PnP framework has been extensively evaluated in two challenging inverse problem settings: super-resolution and joint optimisation for demosaicing and deconvolution. Super-resolution tests the model's ability to enhance image detail across varying upscaling factors (2$\times$, 4$\times$ and 8$\times$), while the joint optimisation task emphasises its robustness and adaptability by tackling the intertwined challenges of demosaicing and deconvolution. The latter setting is particularly demanding due to the need to simultaneously address two operations—removing the effects of mosaic patterns while restoring the high-frequency details lost during image acquisition—requiring a  balance of accuracy and computational.

\subsubsection{Yes, You KAN do Super-resolution Better!} 
The results in Table~\ref{SR} underscore the significant advantages of KAN-PnP over a range of competing single-shot techniques, including INR, FFDNet, and UNet, across all scaling factors and datasets. KAN-PnP demonstrates a clear ability to generalise across diverse image types while maintaining robust performance under varying levels of degradation. Its superior adaptability, stability, and capacity to preserve intricate details highlight why it consistently outperforms other methods, particularly in challenging scenarios.

At 2x upscaling, KAN-PnP reports exceptional performance across all samples. For example, in the Dog instance, KAN-PnP achieves a PSNR of 25.82, outperforming FFDNet (24.02) and UNet (18.71) by notable margins. In \textit{texture-heavy} samples like Squirrel and Koala, KAN-PnP achieves PSNR values of 27.07 and 25.00, respectively, further demonstrating its capacity to handle both structured and unstructured features. While INR is a close competitor in these cases, KAN-PnP consistently either matches or exceeds its performance, solidifying its position as a leading single-shot method.
At 4x scaling, the differences between KAN-PnP and its competitors become even more pronounced. For example, on the Bird sample, KAN-PnP achieves a PSNR of 22.19, surpassing both FFDNet (21.86) and INR (22.06). UNet struggles significantly at this scale, achieving only 16.53, indicating its inability to handle moderate degradation effectively. Similarly, in instances like Fox and Giraffe, KAN-PnP maintains a distinct performance edge, with PSNR values of 23.29 and 24.24, respectively, compared to UNet's 14.38 and 14.12. These results emphasise KAN-PnP's ability to balance global structure recovery and fine detail preservation, a challenge that often undermines the performance of its competitors. 

\begin{table}[t!]
\caption{The performance (PSNR(dB)) comparison of different methods on super-resolution (SR) tasks with $2 \times$, $4 \times$, and $8 \times$ upscaling factors.}\label{SR}%
\renewcommand{\arraystretch}{1.4}
\setlength{\tabcolsep}{5.9pt}
\begin{tabular}{l|c|c|c|c|c|c|c|c|c}
\Xhline{0.25ex}
Method        & \multicolumn{3}{c|}{Dog}                           & \multicolumn{3}{c|}{Fox}                           & \multicolumn{3}{c}{Giraffe}                        \\ \hhline{~---------}
              & 2$\times$      & 4$\times$      & 8$\times$      & 2$\times$      & 4$\times$      & 8$\times$      & 2$\times$      & 4$\times$      & 8$\times$      \\ \hline
FFDNet-PnP & 24.02          & 21.98          & 14.88          & 25.10          & 23.19          & 6.51           & 4.29           & 13.20          & 17.44          \\ \hline
UNet-PnP   & 18.71          & 14.34          & 6.86           & 19.82          & 14.38          & 14.99          & 17.81          & 14.12          & 6.31           \\ \hline
INR-PnP      & 25.80 & 22.00 & \textbf{15.40}          & 26.73          & 23.24 & \textbf{16.21} & 27.93          & 24.23          & 17.46          \\ \hline
\cellcolor[HTML]{D7FFD7}KAN-PnP  & \cellcolor[HTML]{D7FFD7}\textbf{25.82}          & \cellcolor[HTML]{D7FFD7}\textbf{22.05}          & \cellcolor[HTML]{D7FFD7}\textbf{15.40} & \cellcolor[HTML]{D7FFD7}\textbf{27.06} & \cellcolor[HTML]{D7FFD7}\textbf{23.29}         & \cellcolor[HTML]{D7FFD7}\textbf{16.21}          & \cellcolor[HTML]{D7FFD7}\textbf{27.94} & \cellcolor[HTML]{D7FFD7}\textbf{24.24} & \cellcolor[HTML]{D7FFD7}\textbf{17.47} \\ 

\Xhline{0.25ex}
\end{tabular}%

\vspace{0.3cm}
\begin{tabular}{l|c|c|c|c|c|c|c|c|c}
\Xhline{0.25ex}
Method        & \multicolumn{3}{c|}{Koala}                         & \multicolumn{3}{c|}{Squirrel}                       & \multicolumn{3}{c}{Bird}                           \\ \hhline{~---------}
              & 2$\times$      & 4$\times$      & 8$\times$      & 2$\times$      & 4$\times$      & 8$\times$      & 2$\times$      & 4$\times$      & 8$\times$      \\ \hline
FFDNet-PnP & 9.02           & 1.09           & 14.58          & 4.86           & 4.86           & 16.34          & 26.34          & 21.86          & 15.71          \\ \hline
UNet-PnP   & 14.24          & 13.46          & 2.22           & 17.14          & 13.66          & 3.37           & 21.82          & 16.53          & 10.69          \\ \hline
INR-PnP     & 24.97          & 20.51          & \textbf{14.79}          & \textbf{27.08} & \textbf{23.05}          & \textbf{16.93}          & \textbf{26.44} & 22.06          & 15.74          \\ \hline
\cellcolor[HTML]{D7FFD7}KAN-PnP  & \cellcolor[HTML]{D7FFD7}\textbf{25.00} & \cellcolor[HTML]{D7FFD7}\textbf{20.56} & \cellcolor[HTML]{D7FFD7}\textbf{14.79} & \cellcolor[HTML]{D7FFD7}27.07          & \cellcolor[HTML]{D7FFD7}\textbf{23.05} & \cellcolor[HTML]{D7FFD7}\textbf{16.93} & \cellcolor[HTML]{D7FFD7}26.43          & \cellcolor[HTML]{D7FFD7}\textbf{22.19} & \cellcolor[HTML]{D7FFD7}\textbf{15.75} \\ 
\Xhline{0.25ex}
\end{tabular}%
\end{table}
\begin{figure}[!t]
\centering
\includegraphics[width=\textwidth]{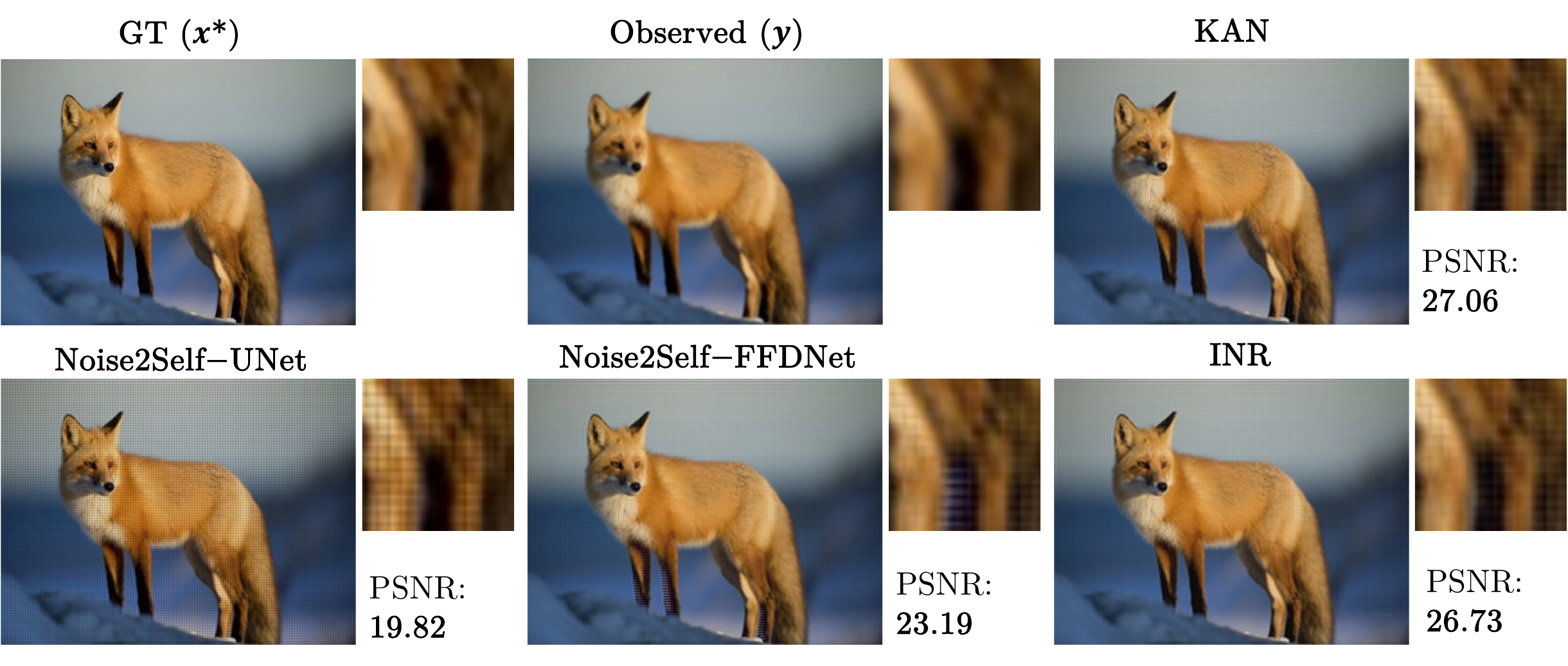}
\caption{The visualisation comparison of the `Fox' example on 2$\times$ Super-Resolution task among the UNet, FFDNet, INR and KAN denoising prior in single-shot setting.}\label{SR2}
\end{figure}

\begin{figure}[t!]
\centering
\includegraphics[width=\textwidth]{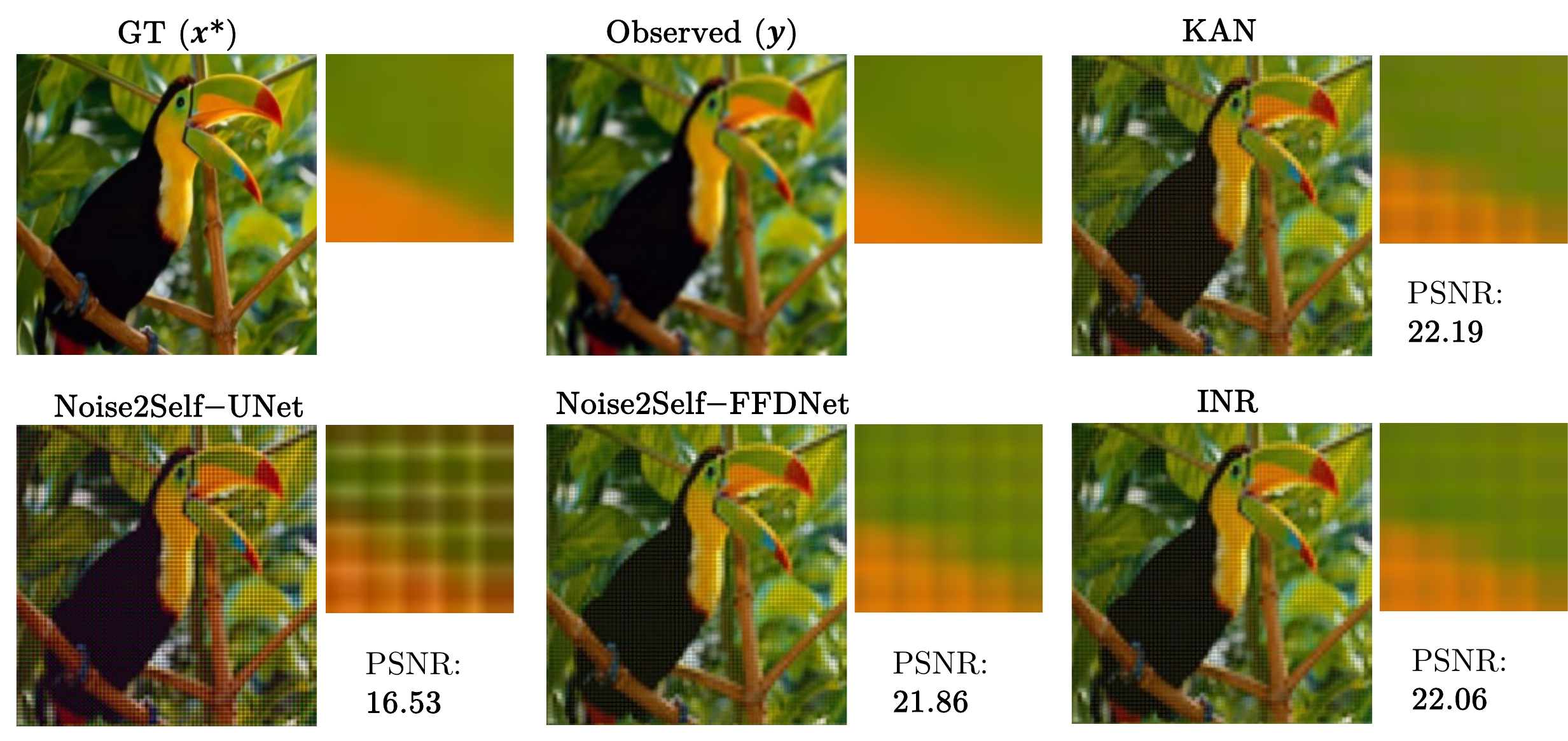}
\caption{Comparing the visual result of the `Bird' example on 4$\times$ Super-Resolution task in single-shot setting with the UNet, FFDNet, INR and KAN denoising prior.}\label{SR4}
\end{figure}

\textbf{Extreme Degradation:} At 8x scaling, the task of image reconstruction becomes particularly challenging, as the loss of high-frequency information is severe. While the margins between KAN-PnP and INR narrow in some datasets (e.g., Giraffe, where KAN achieves 17.47 compared to INR’s 17.46), KAN-PnP demonstrates greater consistency across diverse datasets and excels in preserving fine details. On the Bird dataset at 8×, KAN-PnP achieves a PSNR of 15.75, outperforming FFDNet (15.71) and significantly exceeding UNet’s 10.69. This consistency across datasets highlights the robustness of KAN-PnP in scenarios where competing methods often struggle to generalise or maintain stability. Additionally, KAN-PnP’s architectural design allows it to preserve critical image features such as edges and textures, which are frequently lost in reconstructions by FFDNet and UNet. Particularly and unlike INR, KAN-PnP remains stable under extreme degradation (e.g., 8×) while excelling in texture-heavy and structured datasets alike. 

\textbf{Visual Results.} We further support our numerical results through a set of visual outputs, which are displayed in Figures~\ref{SR2} and~\ref{SR4}. KAN-PnP demonstrates remarkable detail preservation, particularly in the texture of the fox’s fur and the sharpness of its silhouette, as evidenced in the zoomed-in patches. In contrast, INR, despite a close PSNR, exhibits slight edge blurring, while FFDNet oversmooths the background, and UNet introduces grid-like distortions. In the more challenging 4x example (Figure~\ref{SR4}, Bird), KAN-PnP maintains its edge, marginally surpassing INR  and outperforming FFDNet and UNet. KAN-PnP excels in preserving edge sharpness and texture, as seen in the bird’s feathers and beak, while retaining vibrant colour fidelity. INR, while competitive, shows blurring of edges, whereas FFDNet produces oversmoothed results, and UNet introduces significant artefacts and detail loss. These results reaffirm KAN-PnP's ability to deliver visually superior reconstructions with sharper details, fewer artefacts, and better preservation of textures and edges across varying scaling factors, setting it apart as a robust single-shot super-resolution method.

\begin{table}[t!]
\caption{The performance (PSNR(dB), SSIM) comparison on joint deconvolution and demosaicing of single-shot deep denoising priors (KAN, INR, Noise2Self-UNET, and Noise2Self-FFDNet), pre-trained deep denoising priors following~\cite{lai2023prox} (Pretrained-UNet and Pretrained-FFDNet), and classical denoising priors (TV) in Plug-and-Play framework.}\label{joint}%
\renewcommand{\arraystretch}{1.4}
\setlength{\tabcolsep}{4.7pt}

\begin{tabular}{l|l|c|c|c|c|c|c|c|c}
\Xhline{0.25ex}
                  \multicolumn{1}{l}{}         &        & \multicolumn{2}{c|}{Raccoon}                          & \multicolumn{2}{c|}{Fractals}                        & \multicolumn{2}{c|}{Wolf}                           & \multicolumn{2}{c}{Butterfly}                      \\ \hhline{~~--------}
                  \multicolumn{1}{l}{}         &        & PSNR           & SSIM          & PSNR           & SSIM          & PSNR           & SSIM          & PSNR           & SSIM          \\ \hline
{Classic}                      & TV     & {15.55}          & 0.43          & {13.08}          & 0.24          & {17.27}          & 0.45          & {14.16}          & 0.29          \\ \hline
{\multirow{2}{*}{Pretrain}}                             & UNet   & {16.48}          & 0.42          & {15.53}          & 0.32          & {17.54}          & 0.36          & {15.43}          & 0.28          \\ \hhline{~---------} 
{}                             & FFDNet & {23.27}          & 0.75          & {20.26}          & 0.57          & {25.66}          & 0.81          & {20.26}          & 0.50          \\ \hline
{\multirow{4}{*}{\begin{tabular}[c]{@{}l@{}}Single\\ Instance\end{tabular}}}       & Noise2Self-UNet & {14.93}          & 0.23          & {11.85}          & 0.16          & {18.45}          & 0.28          & {13.03}          & 0.20          \\                       \hhline{~---------} 
{}                             & Noise2Self-FFDNet & {25.12}          & 0.79          & {22.77}          & 0.80          & {30.54}          & 0.90          & {23.86}          & 0.84          \\ \hhline{~---------} 
{}   & INR & {25.63}          & 0.82          & {22.86}          & 0.80          & {30.90}          & \textbf{0.91}          & {23.87}          & \textbf{0.86}          \\ \hhline{~---------}  
{}                             & \cellcolor[HTML]{D7FFD7}KAN & {\cellcolor[HTML]{D7FFD7}\textbf{25.65}} & \cellcolor[HTML]{D7FFD7}\textbf{0.82} & {\cellcolor[HTML]{D7FFD7}\textbf{24.46}} & \cellcolor[HTML]{D7FFD7}\textbf{0.81} & {\cellcolor[HTML]{D7FFD7}\textbf{31.02}} & \cellcolor[HTML]{D7FFD7}\textbf{0.91} & {\cellcolor[HTML]{D7FFD7}\textbf{24.00}} & \cellcolor[HTML]{D7FFD7}\textbf{0.86} \\ \Xhline{0.25ex}
\end{tabular}%
%
\vspace{0.3cm}
\begin{tabular}{l|l|c|c|c|c|c|c|c|c}
\Xhline{0.25ex}
         \multicolumn{1}{l}{}                                         &        & \multicolumn{2}{c|}{Giraffe}                      & \multicolumn{2}{c|}{Mushroom}                     & \multicolumn{2}{c|}{Baby}                         & \multicolumn{2}{c}{Monarch}                      \\ \hhline{~~--------} 
        \multicolumn{1}{l}{}                                         &        & PSNR           & SSIM          & PSNR           & SSIM          & PSNR           & SSIM          & PSNR           & SSIM          \\ \hline
{Classic}                      & TV     & {13.82}          & 0.21         & {17.83}          & 0.53         & {13.13}          & 0.33         & {16.35}          & 0.41         \\ \hline
{\multirow{2}{*}{Pretrain}}                             & UNet   & {16.71}          & 0.26         & {18.09}          & 0.49         & {17.77}          & 0.50         & {17.10}          & 0.41         \\ \hhline{~---------} 
{}                             & FFDNet & {25.75}          & 0.73         & {25.02}          & 0.82         & {25.98}          & 0.80         & {21.82}          & 0.63         \\ \hline
{\multirow{4}{*}{\begin{tabular}[c]{@{}l@{}}Single\\ Instance\end{tabular}}}       & Noise2Self-UNet & {12.62}          & 0.11         & {18.78}          & 0.34         & {12.95}          & 0.13         & {15.39}          & 0.26         \\                      \hhline{~---------}  
{}                            & Noise2Self-FFDNet & {29.15}          & 0.83         & {26.86}          & 0.88         & {28.32}          & 0.83         & {26.31}          & \textbf{0.89}         \\   \hhline{~---------} 
{}   & INR & {30.64}          & \textbf{0.89}         & {26.99}          & \textbf{0.89}         & {29.09}          & 0.86         & {26.48}          & \textbf{0.89}         \\ \hhline{~---------} 
{}                             & \cellcolor[HTML]{D7FFD7}KAN & {\cellcolor[HTML]{D7FFD7}\textbf{30.80}} & \cellcolor[HTML]{D7FFD7}\textbf{0.89} & {\cellcolor[HTML]{D7FFD7}\textbf{27.00}} & \cellcolor[HTML]{D7FFD7}\textbf{0.89} & {\cellcolor[HTML]{D7FFD7}\textbf{29.27}} & \cellcolor[HTML]{D7FFD7}\textbf{0.87} & {\cellcolor[HTML]{D7FFD7}\textbf{26.52}} & \cellcolor[HTML]{D7FFD7}\textbf{0.89} \\ \Xhline{0.25ex}
\end{tabular}%

\end{table}

\subsubsection{Joint Restoration Tasks} 
The results in Table~\ref{joint} highlight the effectiveness of KAN-PnP when applied to the challenging joint task of demosaicing and deconvolution, demonstrating its superiority over classical, pre-trained, and other single-shot methods. These tasks compound the complexity of inverse problems, as they require the model to handle two distinct operations—recovering high-frequency details lost during the acquisition process and simultaneously reconstructing colour patterns. Despite these challenges, KAN-PnP consistently delivers the highest or comparable Peak Signal-to-Noise Ratio (PSNR) and Structural Similarity Index (SSIM) scores across all datasets, showcasing its robustness and adaptability.

\begin{figure}[!t]
\centering
\includegraphics[width=\textwidth]{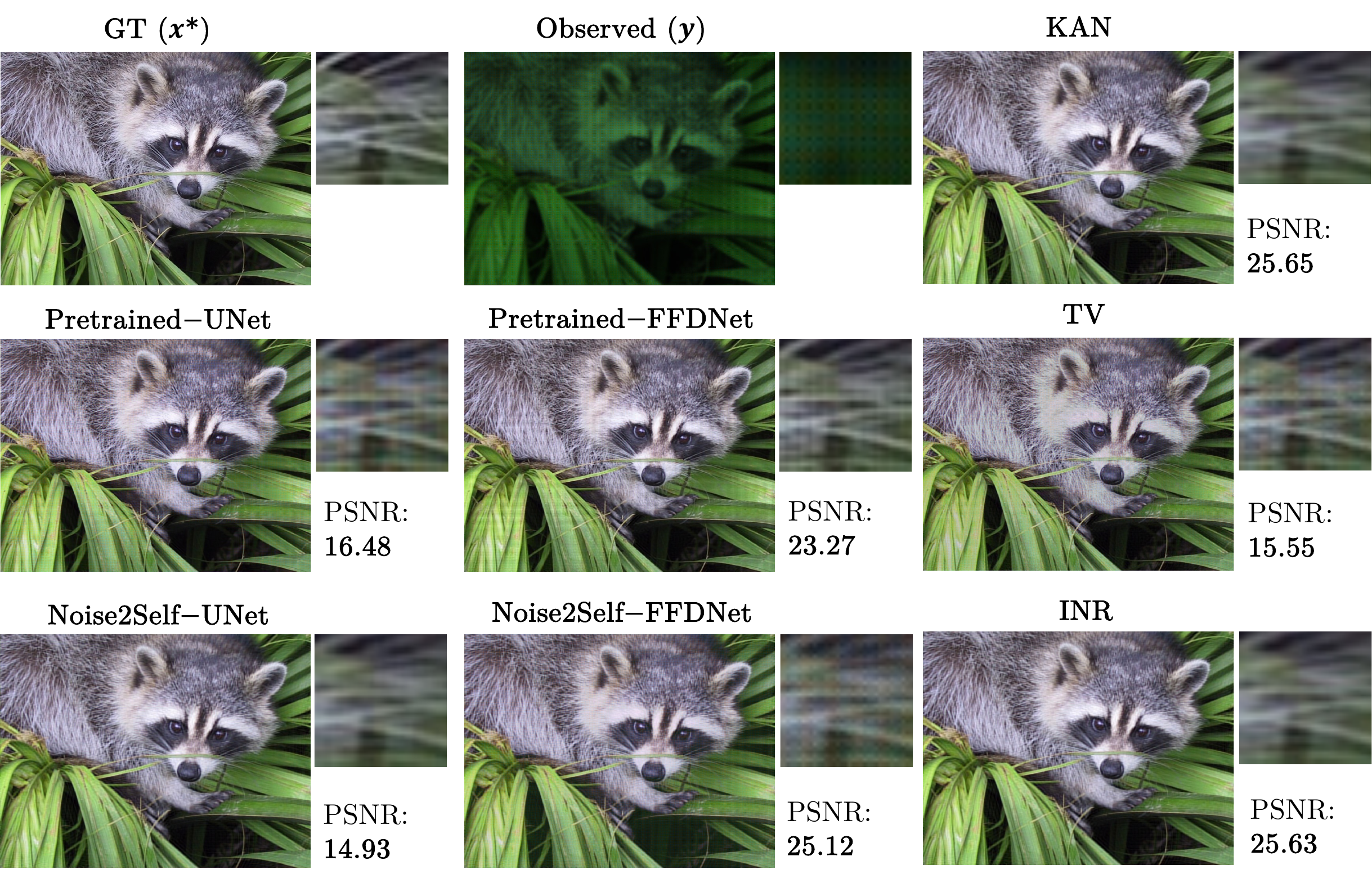}
\caption{Comparison of the visual output of the Plug-and-Play framework of the `Raccoon' example on Joint Deconvoution and Demosaicing task with traditional prior (TV), pretrained prior (UNet and FFDNet) and Single-Shor prior (Noise2Self-UNet, Noise2Self-FFDNet, INR and KAN).}\label{joint-raccoon}
\end{figure}
\begin{figure}[t!]
\centering
\includegraphics[width=\textwidth]{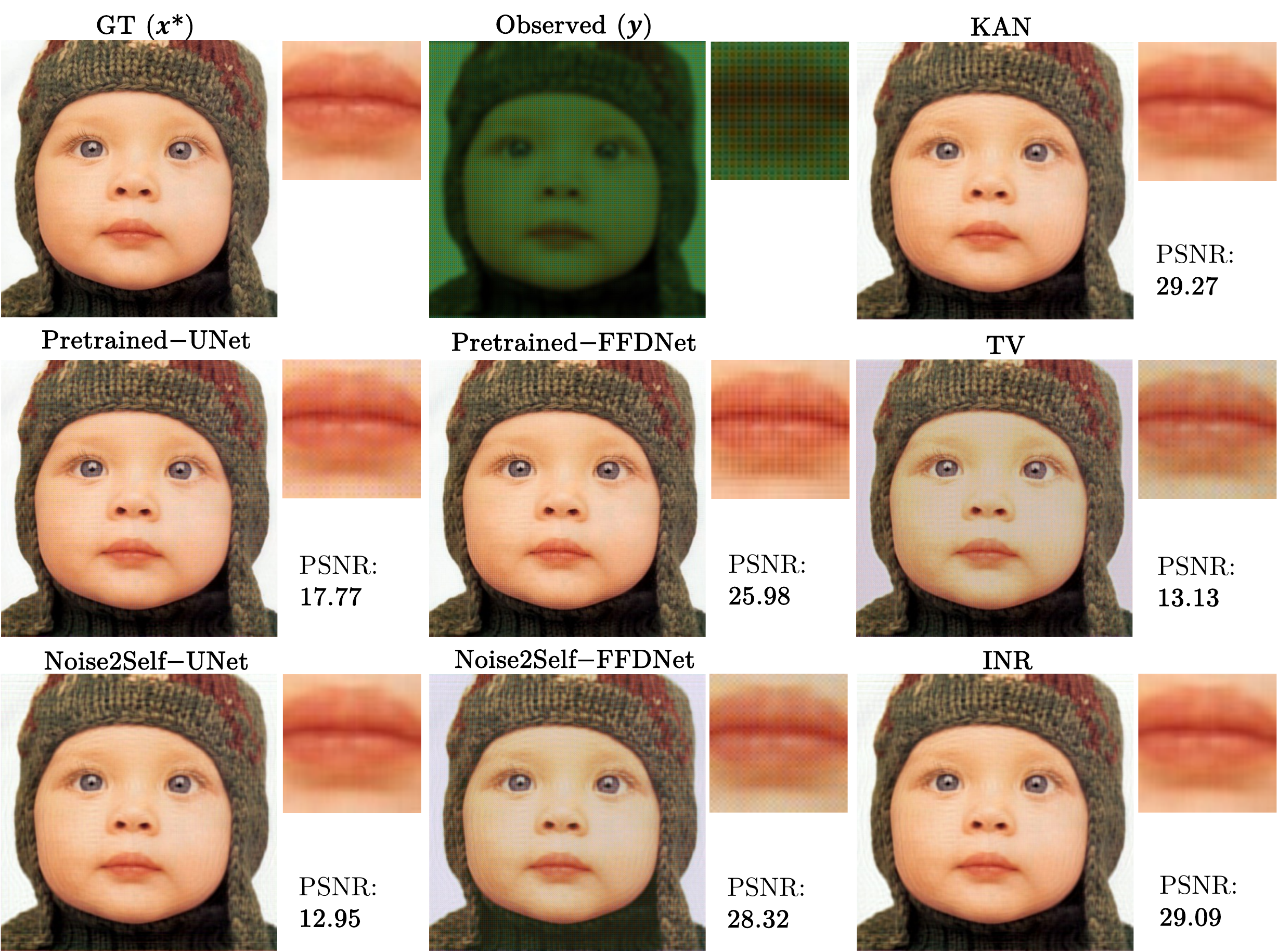}
\caption{The visualisation comparison of the `Baby' example on Joint Deconvoution and Demosaicing task with traditional prior (TV), pretrained prior (UNet and FFDNet) and single-shot prior (Noise2Self-UNet, Noise2Self-FFDNet, INR and KAN) in the Plug-and-Play framework.}\label{joint-baby}
\end{figure}

KAN-PnP outperforms its competitors, including INR and Noise2Self-FFDNet, in both PSNR and SSIM metrics. For example, on the Wolf dataset, KAN achieves a PSNR of 31.02 and an SSIM of 0.91, surpassing the already strong performance of INR (30.90, 0.91) and Noise2Self-FFDNet (30.54, 0.90). Similarly, for the Butterfly dataset, KAN achieves 24.00 PSNR and 0.86 SSIM, outperforming INR (23.87, 0.86) and significantly exceeding pre-trained methods like FFDNet (20.26, 0.50). These results underscore the ability of KAN-PnP to preserve fine details, edges, and textures while accurately reconstructing complex color patterns, making it particularly suited for high-dimensional inverse problems.

KAN-PnP's superior performance can be attributed to its mathematical foundation based on the Kolmogorov-Arnold representation, which ensures smooth, stable approximations while preserving fine details and structures. The Lipschitz continuity of the KAN denoiser ensures stability during optimisation, a critical factor for solving complex inverse problems like joint demosaicing and deconvolution. By effectively balancing local and global features, KAN-PnP surpasses both classical and deep learning-based methods, even in scenarios with limited data availability.

The visualisations in Figures~\ref{joint-raccoon} and~\ref{joint-baby} further support these numerical results, showcasing KAN-PnP’s ability to deliver sharper reconstructions and preserve critical details compared to other methods. For instance, in the Raccoon example (Figure~\ref{joint-raccoon}), KAN-PnP produces the most visually accurate result, effectively restoring the texture of the fur and the structure of the leaves, while methods like Noise2Self-FFDNet and TV fail to maintain texture fidelity and introduce artifacts. Similarly, in the Baby example (Figure~\ref{joint-baby}), KAN-PnP achieves a PSNR of 29.27, recovering fine features such as the softness of the lips and smoothness of the skin, outperforming pre-trained methods like FFDNet, which oversmooths details, and classical methods like TV, which exhibit heavy artifacts. These visual comparisons highlight KAN-PnP’s ability to consistently produce high-quality reconstructions that align with its strong numerical performance.

\subsubsection{Unplugging KAN: Ablation Study}
\fcircle[fill=deblue]{2pt}\textbf{ KAN Denoiser Basis. }The results in Table~\ref{ablation-KAN} provide valuable insights into the impact of different basis functions used in the KAN denoiser on image reconstruction performance. Three basis functions—B-Spline, Fourier, and Symbolic—are evaluated across multiple datasets, with metrics presented in terms of Peak Signal-to-Noise Ratio (PSNR) and Structural Similarity Index (SSIM). These findings allow us to better understand the contributions of the selected basis to the overall reconstruction quality and robustness.
This consistent superiority highlights the strength of B-Spline's smooth and adaptive nature, which is particularly effective at capturing local features and preserving fine textures, such as those in the Wolf and Butterfly datasets. The smoothness of B-Spline functions enables it to balance noise suppression and detail recovery, making it an ideal choice for complex reconstruction tasks. The Fourier basis shows a notable decline in performance compared to B-Spline and Symbolic bases, particularly on texture-rich datasets such as Raccoon and Fractals. The limitations of the Fourier basis may stem from its sensitivity to discontinuities and noise, which can hinder its ability to reconstruct images with complex structures or textures. While Symbolic performs well overall, it tends to lag behind B-Spline on datasets requiring higher levels of smoothness and adaptability, such as the Wolf and Butterfly samples.

\begin{table}[t!]
\caption{The performance (PSNR(dB) and SSIM) comparison across different KAN layers (B-Spline, Fourier, and Symbolic) for image reconstruction.}\label{ablation-KAN}%
\renewcommand{\arraystretch}{1.4}
\setlength{\tabcolsep}{6.2pt}
\begin{tabular}{l|c|c|c|c|c|c|c|c}
\Xhline{0.25ex}
{Method} & {Metric} & \,Raccoon\,     & Fractals   & Cartoon    & \,\, Wolf \, \quad       & Butterfly  & \,\, Dog\,\, \quad        & Forest     \\ \hline
\multirow{2}{*}{B-Spline}   & PSNR        & \textbf{25.65}      & \textbf{24.46}      & \textbf{23.73}      & \textbf{31.02} & \textbf{24.00} & \textbf{28.80}      & \textbf{20.97} \\ \hhline{~--------} 
                            & SSIM        & \textbf{0.82}       & \textbf{0.81} & \textbf{0.79} & \textbf{0.91} & \textbf{0.86} & \textbf{0.89} & \textbf{0.64} \\ \hline
\multirow{2}{*}{Fourier}    & PSNR        & 23.29      & 20.82      & 19.78      & 23.37      & 19.70      & 23.77      & 18.90      \\ \hhline{~--------}
                            & SSIM        & 0.71       & 0.67       & 0.54       & 0.69       & 0.64       & 0.69       & 0.51       \\ \hline
\multirow{2}{*}{Symbolic}   & PSNR        & 25.48      & 24.40      & 23.51      & 30.58      & 23.86      & 28.29      & 20.91      \\ \hhline{~--------}
                            & SSIM        & 0.81       & 0.80       & 0.77       & 0.90       & 0.84       & 0.87       & \textbf{0.64} \\
\Xhline{0.25ex}
\end{tabular}%
\end{table}

\begin{table}[t!]
\caption{The comparison with other methods that using few data (BM3D-PnP and DIP-PnP) with the KAN prior in the Plug-and-Play framework on different number of iterations.}\label{ablation-other}%
\renewcommand{\arraystretch}{1.4}
\begin{tabular}{l|l|c|c|c|c|c|c|c|c}
\Xhline{0.25ex}
        \multicolumn{1}{l}{}                     &         & \multicolumn{2}{c|}{Koala}                           & \multicolumn{2}{c|}{Squirrel}                        & \multicolumn{2}{c}{Bird}   & \multicolumn{2}{c}{Fractals}                             \\ \hhline{~~--------}  
        \multicolumn{1}{l}{}                      &         & {2$\times$}      & 4$\times$      & {2$\times$}      & 4$\times$      & {2$\times$}      & 4$\times$   & {2$\times$}      & 4$\times$      \\ \hline
{\multirow{2}{*}{100 iter}} & BM3D-PnP    & {9.06}           & 8.36           & {16.04}          & 15.05          & {11.42}          & 10.87      &16.07   &11.56    \\ \hhline{~---------}   
{}                          & DIP-PnP & {14.81}          & 14.55          & {18.57}          & 18.01          & {16.31}          & 15.81      &16.14   &16.38    \\ \hline
{\multirow{2}{*}{200 iter}} & BM3D-PnP    & {9.13}           & 8.54           & {16.91}          & 15.70          & {11.66}          & 11.35      &17.07   &12.21    \\ \hhline{~---------} 
{}                          & DIP-PnP & {21.57}          & 18.39          & {23.84}          & 20.70          & {22.18}          & 20.06      &21.58   &18.45    \\ \hline
{5 iter}                    & \cellcolor[HTML]{D7FFD7}KAN-PnP     & {\cellcolor[HTML]{D7FFD7}\textbf{25.00}} & \cellcolor[HTML]{D7FFD7}\textbf{20.56} & {\cellcolor[HTML]{D7FFD7}\textbf{27.07}} & \cellcolor[HTML]{D7FFD7}\textbf{23.05} & {\cellcolor[HTML]{D7FFD7}\textbf{26.43}} & \cellcolor[HTML]{D7FFD7}\textbf{22.19}  & {\cellcolor[HTML]{D7FFD7}\textbf{25.47}} & \cellcolor[HTML]{D7FFD7}\textbf{21.18} \\ \Xhline{0.25ex}

\end{tabular}
\end{table}

\fcircle[fill=deblue]{2pt}\textbf{ KAN-PnP against other methods with fewer data. } The results in Table~\ref{ablation-other} compare KAN-PnP with other schemes such as BM3D-PnP, which relies on classical denoisers with minimal data, and DIP-PnP, which uses Deep Image Prior to learn priors with small datasets (typically around 15 images). This comparison highlights KAN-PnP’s ability to converge faster and achieve superior performance, even in settings where data is extremely limited. While DIP-PnP follows a philosophy closer to KAN-PnP than BM3D-PnP, its reliance on small datasets for training the denoiser sets it apart, emphasising KAN-PnP’s unique ability to perform robustly with a single-instance prior. KAN-PnP consistently achieves the highest PSNR values across all datasets and scaling factors, outperforming both BM3D-PnP and DIP-PnP. These results highlight the strength of KAN-PnP in reconstructing high-quality images with a single-instance prior, where BM3D-PnP struggles with limited flexibility, and DIP-PnP is constrained by its dependence on small training datasets.

KAN-PnP converges in just 5 iterations, demonstrating remarkable efficiency compared to the 100 or 200 iterations required by BM3D-PnP and DIP-PnP. This rapid convergence is a significant advantage, especially in scenarios where computational resources or time are limited. The ability to achieve superior performance in fewer iterations can be attributed to the Kolmogorov-Arnold Network (KAN), which provides a mathematically grounded, adaptive framework that captures intricate features and textures efficiently.

\section{Conclusion}
In this work, we introduced KAN-PnP, an optimisation framework that integrated Kolmogorov-Arnold Networks (KANs) as denoisers within the Plug-and-Play (PnP) paradigm. KAN-PnP was specifically designed to tackle inverse problems with single-instance priors, where only a single noisy observation was available, thus eliminating the need for large datasets typically required by traditional denoising methods. We demonstrated that KANs, based on the Kolmogorov-Arnold representation theorem, effectively served as priors in these settings, offering a robust solution for denoising.
We provided theoretical guarantees for KAN-PnP, proving its convergence under key assumptions. Extensive experimental results showed that KAN-PnP outperformed existing denoising methods, providing superior performance in single-shot learning scenarios with minimal data. The method exhibited reliable convergence, achieving high accuracy with fewer iterations. Comparisons with state-of-the-art approaches further highlighted KAN-PnP's advantages in both convergence speed and solution quality, demonstrating its potential for solving real-world inverse problems with limited data.
An interesting direction for future work lies in explicitly enforcing strict non-expansiveness of the KAN denoiser (i.e., ensuring \( L_H \leq 1 \)) through architectural constraints. This could lead to tighter theoretical convergence guarantees and further improve the empirical stability of Plug-and-Play methods in the single-instance setting.

\section*{Acknowledgements}
YC is funded by an AstraZeneca studentship and a Google studentship.
CBS acknowledges support from the Philip Leverhulme Prize, the Royal Society Wolfson Fellowship, the EPSRC advanced career fellowship EP/V029428/1, EPSRC grants EP/S026045/1 and EP/T003553/1, EP/N014588/1, EP/T017961/1, the Wellcome Innovator Awards 215733/Z/19/Z and 221633/Z/20/Z, the European Union Horizon 2020 research and innovation programme under the Marie Skodowska-Curie grant agreement No. 777826 NoMADS, the Cantab Capital Institute for the Mathematics of Information and the Alan Turing Institute. 
AIAR gratefully acknowledges the support from Yau Mathematical Sciences Center, Tsinghua University.

\bibliography{sn-bibliography}

\end{document}